\RequirePackage[svgnames,table]{xcolor}
\documentclass[11pt,letterpaper,logo]{yalearxiv}
\usepackage{natbib}
\setcitestyle{square}
\pretocmd{\abscontent}{\vspace{-1mm}\begin{center}\section*{\abstractname}\end{center}}{}{}
\DeclareFontShape{T1}{cmss}{b}{n}{<->ssub * cmss/bx/n}{}
\usepackage[utf8]{inputenc}
\usepackage[T1]{fontenc}
\usepackage{xurl}
\usepackage{caption}
\usepackage{amsmath,amssymb,amsthm,mathtools}
\usepackage{microtype}
\usepackage{booktabs,longtable,array}
\usepackage{algorithm}
\usepackage[noend]{algpseudocode}
\usepackage{graphicx,xcolor}
\usepackage{placeins}
\usepackage{wrapfig}
\hypersetup{colorlinks=true,allcolors=blue!50!black}
\usepackage[nameinlink,capitalise]{cleveref}
\usepackage{adjustbox}
\usepackage{etoc}
\tcbuselibrary{skins,breakable}
\definecolor{promptblue}{HTML}{244763}
\newtcolorbox{promptbox}[1]{
  enhanced,breakable,colback=promptblue!3!white,colframe=promptblue!70!white,
  colbacktitle=promptblue,coltitle=white,fonttitle=\sffamily\bfseries\small,
  fontupper=\small\sffamily,boxrule=.45pt,arc=1.5mm,
  left=3mm,right=3mm,top=2mm,bottom=2mm,toptitle=1.4mm,bottomtitle=1.4mm,
  before skip=9pt,after skip=9pt,title={#1},title after break={#1 (continued)},
  before upper={\raggedright\setlength{\parindent}{0pt}\setlength{\parskip}{4pt}}
}

\makeatletter
\def\theHALG@line{\thealgorithm.\arabic{ALG@line}}
\makeatother
\crefname{equation}{Eq.}{Eqs.}\Crefname{equation}{Eq.}{Eqs.}
\crefname{section}{Section}{Sections}\crefname{appendix}{Appendix}{Appendices}
\crefname{figure}{Figure}{Figures}\crefname{table}{Table}{Tables}
\hypersetup{pdftitle={Reward Engine: Reward-Centric Self-Evolution for Open-Ended Tool Agents}, pdfauthor={Litong Feng, Ji Qi, Yiran Cheng, Shaowu Pan, Xinyu Zhang, Qijun Li}}
\newcommand{\Info}{\mathcal{I}}
\newcommand{\Memory}{\mathcal{M}}
\DeclareMathOperator{\Build}{Build}
\DeclareMathOperator{\Attack}{Attack}
\DeclareMathOperator{\Repair}{Repair}
\DeclareMathOperator{\Critic}{Critic}
\DeclareMathOperator{\Propose}{Propose}
\DeclareMathOperator{\Select}{Select}
\newtheorem{proposition}{Proposition}

\crefname{proposition}{Proposition}{Propositions}
\crefname{algorithm}{Algorithm}{Algorithms}
\newcommand{\fitwidth}[1]{\adjustbox{max width=\linewidth}{#1}}
\newcommand{\AlgStage}[2]{\Statex\hspace*{\dimexpr\algorithmicindent*#1\relax}\textit{#2}}
\title{Reward Engine: Reward-Centric\\ Self-Evolution for Open-Ended Tool Agents}
\runningtitle{Reward Engine: Reward-Centric Self-Evolution for Open-Ended Tool Agents}
\makeatletter
\renewcommand\AB@affilsepx{\protect\\\protect\Affilfont}
\makeatother
\author[1,*]{Litong Feng}
\author[1,*,\textdagger]{Ji Qi}
\author[1]{Yiran Cheng}
\author[1]{Shaowu Pan}
\author[1]{Xinyu Zhang}
\author[1]{Qijun Li}
\affil[1]{Huawei openJiuwen Team}
\affil[*]{Equal contribution.}
\begin{document}
\etocdepthtag.toc{mtmain}
\begin{abstract}
Open-ended tool agents can complete the same objective through different tool paths, and their quality depends jointly on the interaction trace, environment state, and final artifacts, so such tasks rarely admit an executable ground-truth reward. The central challenge is therefore to construct a reward from existing experience that exposes the current policy's failures while remaining stationary during optimization. We introduce Reward Engine, which, like a self-motivated student, first works out what success requires and then optimizes toward it. From the task and initial trajectories, it builds a RewardPack of testcases for required correctness and rubrics for graded quality, each executed as a check that returns a judgment with evidence. The frozen RewardPack then guides the agent to revise and select its outputs, and can also serve as a reward for reinforcement learning. On WorkBuddy Bench Office and Claw-Eval, the constructed reward largely agrees with the benchmark graders, not reversing their ordering on 82.3\% and 80.5\% of candidate pairs. Used to revise executions, it raises three-run consistency ($\mathrm{Pass}^{3}$) over Baseline from 45.0\% to 54.0\% and from 78.4\% to 83.7\%.
Code is available at \url{https://github.com/LeibnizFlow/reward-engine}.
\end{abstract}
\maketitle
\begingroup
\renewcommand{\thefootnote}{\textdagger}
\renewcommand{\footnotesize}{\fontsize{8}{10}\selectfont}
\footnotetext[2]{Ji Qi is a PhD student of Department of Statistics and Data Science, The Chinese University of Hong Kong (CUHK).}
\endgroup

\section{Introduction}
\label{sec:introduction}
\begin{figure}[p]
\centering
\includegraphics[width=\textwidth,height=0.76\textheight,keepaspectratio]{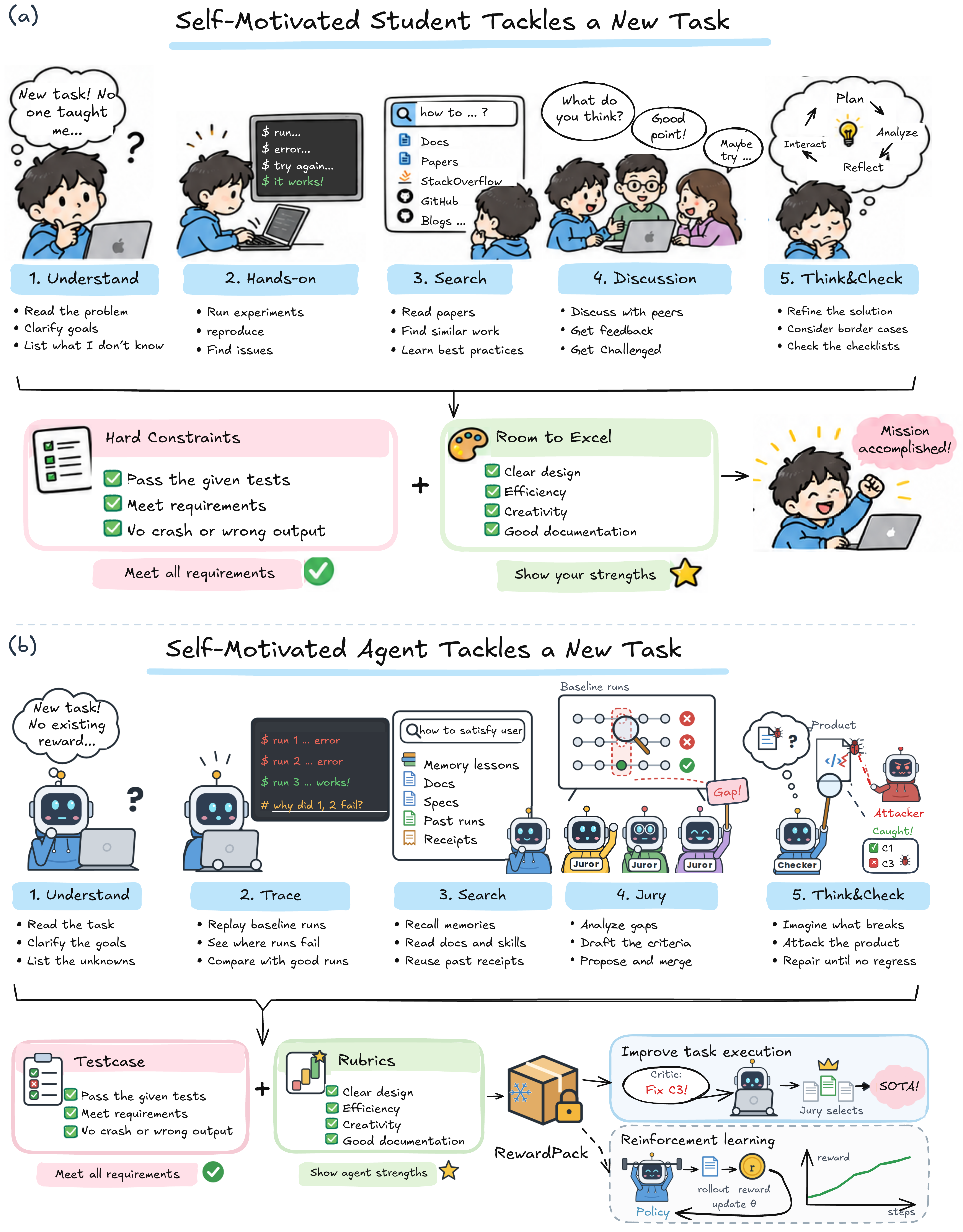}
\caption{From a self-motivated student to a self-motivated agent. (a) A student facing a new task identifies hard constraints and room to excel. (b) Testcases (required correctness) and rubrics (graded quality), bound to executable checks, form a frozen RewardPack. The reinforcement-learning branch is a possible extension and is not evaluated.}
\label{fig:motivation}
\end{figure}

Tool agents are increasingly asked to do open-ended work, such as reconciling spreadsheets, drafting reports from mixed-format files, or triaging tickets \citep{yao2025taubench,workbuddybench2026,ye2026claweval}. Unlike verifiable tasks such as math with final answers or code with unit tests \citep{cobbe2021verifiers,jimenez2024swebench}, whose rewards directly drive reinforcement learning \citep{shao2024deepseekmath}, these tasks come with no reward: many solutions are acceptable, and quality depends on the trace, environment state, and delivered artifacts. Constructing the reward thus becomes the key, and harder, problem; without it, an agent can neither critique and select among its attempts nor train on them. Common proxies fall short: a weekly report whose completion rate reads 92\% although the source rows give 78\% passes a file-existence check and can pass a holistic model judge, whose verdicts depend on presentation and ordering \citep{zheng2023judging,wang2024fair}. Like a self-motivated student facing an untaught task, the agent must first work out which constraints must hold and where there is room to excel (\cref{fig:motivation}), and the resulting reward must meet three requirements: 

\begin{itemize}
    \item \textbf{Evidence-grounded:} each criterion, whether a hard constraint (testcase) or graded quality (rubric), must be judged by recomputing it from the artifact, trace, and environment and returned with a receipt, such as ``expected 78\%, observed 92\%: fail''. The receipt tells a critic what to fix and a selector why one candidate ranks higher, and it keeps a violated requirement apart from an unreadable file or missing evidence.
    \item \textbf{Tested before use:} a drafted check can itself be wrong, for instance by reading the wrong column or accepting a near miss, so the reward must be tested against labeled counterexamples before the agent relies on it, and a repair must not break cases it already judged correctly.
    \item \textbf{Stationary:} the reward must stay fixed while the agent is revised or selected against it and change only between rounds; otherwise an apparent gain may come from a moved target rather than a better output.
\end{itemize}

\begin{figure}[t]
\centering
\includegraphics[width=\textwidth]{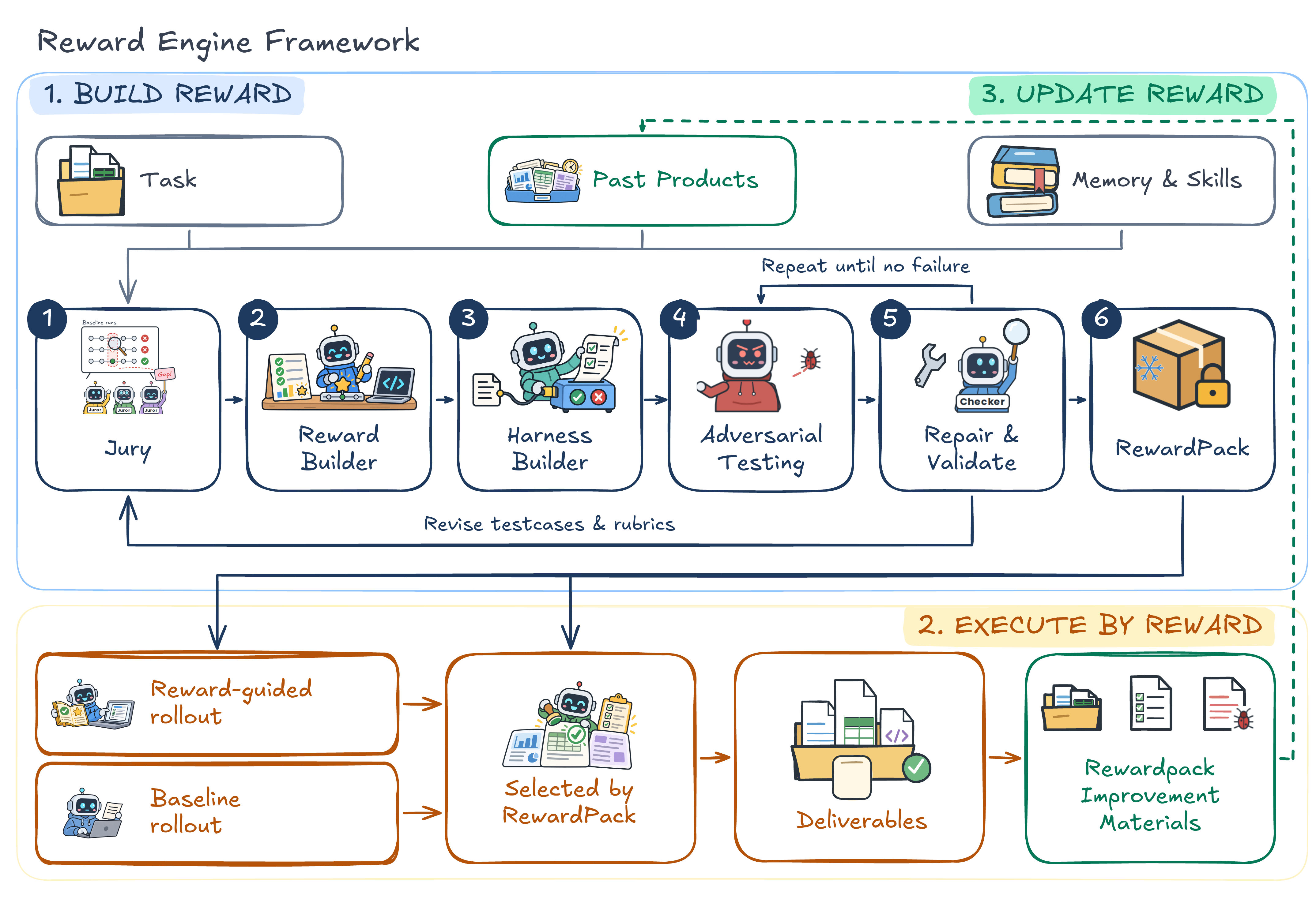}
\caption{Reward Engine in three stages. \emph{1. Build reward (blue):} criteria and executable checks are tested and repaired before a RewardPack is frozen. \emph{2. Execute by reward (orange):} the frozen pack supports revision and selection. \emph{3. Update reward (green):} products, receipts, and use-time defects inform the next construction round. Benchmark scores are reserved for evaluation.}
\label{fig:framework}
\end{figure}

Existing evaluators and reward-construction methods meet these requirements only in part (\cref{tab:related_comparison} in \cref{app:related_work}). Model judges, rubric judges, and trained verifiers \citep{zheng2023judging,kim2024prometheus,saadfalcon2025lmunit,gunjal2026rubrics,cobbe2021verifiers} score outputs, but their judgments are not bound to re-runnable procedures, and at most partially record the evidence behind them \citep{zhuge2025agentjudge,rao2026autorubric}. Reward-construction methods write executable reward code or tests \citep{ma2024eureka,xie2024text2reward,chen2023codet} but keep no evidence receipts, so a violated requirement cannot be told apart from a broken check. Most importantly, none tests the reward for a task on labeled counterexamples before optimizing against it, or accepts a repair only if earlier agreement is preserved; the closest, EvoRubrics \citep{ding2026evorubrics}, adapts its rubrics adversarially but updates them at every training step, so the target moves while it is used. What is missing is a reward that is evidence-grounded, tested against counterexamples before use, and held fixed while the agent is optimized against it.

We introduce \textbf{Reward Engine} (\cref{fig:framework}) to close this gap: it builds such a reward for each task and uses it to improve the agent's deliverable. Each round has three stages:

\textbf{Build reward.} From the task and a few initial trajectories, never hidden answers or benchmark scores, a Jury drafts the task's obligations as testcases or graded rubrics, and a Builder binds each to an executable check that returns a judgment with an evidence receipt (\cref{sec:setup,sec:construct}). A quality-control loop tests the reward on labeled control, error, and boundary cases and accepts only repairs that lose no previously matched case; the result is frozen as a \emph{RewardPack} (\cref{sec:repair}).

\textbf{Execute by reward.} A Critic turns failed criteria into revision feedback, and a Jury ranks the revision with the round's other candidates; a candidate replaces the fresh task-only one only if it ranks strictly higher and fails no test the task-only candidate passes (\cref{sec:use}).

\textbf{Update reward.} Products, receipts, and defects found during use inform only the next RewardPack, so the reward never shifts while it is used; a meta loop across task families, which we do not evaluate, can further refine how rewards are built. On WorkBuddy Bench Office and Claw-Eval, the RewardPack does not reverse the benchmark grader's ordering on 82.3\% and 80.5\% of candidate pairs and raises $\mathrm{Pass}^{3}$ over Baseline from 45.0\% to 54.0\% and from 78.4\% to 83.7\% (\cref{sec:main_results}).

Our main contributions are (1) \textbf{An evidence-grounded RewardPack} of testcases for required correctness and rubrics for graded quality, built from the task and initial trajectories, each criterion and rubric bound to an executable evidence check (\cref{sec:setup,sec:construct}). (2) \textbf{Reward-centric self-improvement}, in which the frozen RewardPack guides Critic-driven revision and Jury selection of the deliverable and is updated only between rounds (\cref{sec:use,fig:framework}). (3) \textbf{Two-benchmark evidence} on WorkBuddy Bench Office and Claw-Eval that the RewardPack largely agrees with the benchmark graders and raises the delivered output's pass rate over Baseline (\cref{sec:experiments}).

\section{Problem setup}
\label{sec:setup}\label{sec:preliminaries}
\paragraph{Tasks and objective.} A task supplies public requirements $x$, including its input files, and a tool environment. An actor $\pi$, a language model in a fixed tool harness, takes actions $a_t\sim\pi(\cdot\mid h_t)$ given its context $h_t$ and delivers a \emph{product} $y$: the artifact, its interaction trace, and the resulting environment state. The task's quality $q(y)$, here the benchmark grader's score, is hidden. The system sees only the public information $\Info$: $x$, updates from the task's own (possibly simulated) user, public background knowledge the task does not forbid, the execution history, and a few initial runs $\mathcal B$ of $\pi$ on $x$ alone. From $\Info$ it must (i) construct a reward $R$ whose judgments track $q$ and (ii) use $R$ to deliver a product $y^\star$ of higher $q$ than $\pi$ attains alone; \cref{sec:main_results} measures (i) by order agreement with the grader and (ii) by the scores and pass rates of $y^\star$. The weights of $\pi$ are never updated: what evolves is the reward and the process that builds it. Hidden answers and grading outputs are excluded by design.

\paragraph{Reward.} A reward binds criteria to checks, and running a check on a product returns a judgment with its evidence:
\begin{equation}
 R=\{(c_i,\kappa_i)\}_{i=1}^{m},\qquad
 R(y)=\bigl(\kappa_i(y;x)\bigr)_{i=1}^{m},\qquad
 \kappa_i(y;x)=(z_i,e_i).
 \label{eq:measurement}
\end{equation}
A criterion $c_i$ is a \emph{testcase} (required correctness, $z_i\in\{\mathsf{pass},\mathsf{fail}\}$) or a \emph{rubric} (graded quality, $z_i\in\{\mathsf{fail},\mathsf{partial},\mathsf{full}\}$ against written anchors); $z_i=\bot$ means no judgment, and several bound tests give a tuple (\cref{app:formal_contracts}). A check $\kappa_i$ is a stored, re-runnable procedure, either code, such as recomputing a rate from source rows, or a model judgment over evidence it first locates, so it need not be deterministic. The \emph{receipt} $e_i$ references the inspected evidence and records the execution status, keeping apart an observed violation (a failing judgment), insufficient evidence ($\bot$, normal status), and a check that did not run ($\bot$, error status). A \emph{RewardPack} stores one version $R_v$ with the records needed to audit and re-run it (\cref{tab:rewardpack_schema}).

$R(y)$ is a vector, not a scalar, so each consumer states its use: the Critic turns failing judgments and receipts into feedback (\cref{eq:feedback}); selection ranks candidates by a Jury that reads artifacts and receipts, and rejects any candidate that fails a test the task-only candidate passes (\cref{eq:select}); and the grader comparison orders candidates by the share of resolved criteria that pass. Reinforcement learning would need a similar reduction and is not evaluated.

\paragraph{Requirements.} The requirements of \cref{sec:introduction} become conditions on $R$. \emph{Evidence-grounded}: every judgment other than $\bot$ carries a receipt naming the parts of $y$ and $x$ it used. \emph{Tested before use}: $R$ is first run on labeled cases $(y_j,i_j,z^\star_j)$, each a product, a targeted criterion, and a reference label assigned blind to $R$; $R$ \emph{agrees} on a case if $z_{i_j}=z^\star_j$, is \emph{unresolved} if $z_{i_j}=\bot$, and \emph{disagrees} otherwise, and a repair must keep every case already agreed on (\cref{sec:repair}). \emph{Stationary}: of the three layers, the actor $\pi$, the reward $R$, and the construction state $\phi$ that builds rewards (\cref{sec:construct}), only $\pi$ is fixed throughout. We index actor steps by $t$, rounds by $n$, reward versions by $v$, and updates of $\phi$ by $k$. Quality control may produce several versions within a round; the last is frozen for the round's execution stage, where feedback changes only the context $h_t$, and revisions of $R$ or $\phi$ take effect only at round boundaries. A gain within a round therefore comes from a better product, not a moved target.

\section{Reward Engine}
\label{sec:method}
\Cref{alg:main} gives the per-task procedure over the three stages of \cref{fig:framework}; \cref{app:method_details} gives the notation, RewardPack schema, and implementation details.

\subsection{Constructing criteria and checks}
\label{sec:construct}
\paragraph{Jury.} Three Jury members draft obligations from the public task before seeing any artifact, then revisit them against the initial task-only runs $\mathcal B$, which show where the current policy falls short and where acceptable solutions differ. A reducer merges the drafts into criteria, each marked as a testcase or a rubric, keeping each criterion's source, applicability, and fail/partial/full anchors; unresolved obligations stay visible. For the weekly report, Jury fixes the reporting period and the completion-rate definition. An obligation that no member proposes remains an omission that this review cannot detect (\cref{app:runtime_details}).

\paragraph{Builder.} Builder turns each criterion into executable tests, refining its fail/partial/full anchors into the rubric a model-judge check applies (Reward Builder), and binds them to inputs, tools, and judgment procedures (Harness Builder); for the weekly report, the completion-rate test recomputes the rate from the source rows. Construction is guided by the \emph{construction state} $\phi=(\omega,\Memory)$, which pairs Strategy $\omega$ with Memory $\Memory$, under a scope, delivery contract, and budget:
\begin{equation}
 R_0\sim\Build_\phi(\Info),
 \label{eq:build_operator}
\end{equation}
where $\Build$ may instead return an explicit construction failure. Each receipt of \cref{eq:measurement} is bound to the inspected artifact, environment snapshot, inputs, scope, and check version, so a judgment can be audited, and a receipt is reused only when its artifact, environment snapshot, scope, and check version all still match (\cref{app:formal_contracts}).

\paragraph{Strategy and Memory.} Strategy $\omega$ holds module guidance and settings for quality control (QC, \cref{sec:repair}), and Memory $\Memory$ stores construction lessons with their sources and validation status. Stage 3 of every round (\cref{fig:framework}) adds lessons to Memory, so Memory grows within the inner loop, whereas Strategy changes only in the meta loop of \cref{alg:main}. That loop proposes updates of $\phi$ on training task families and selects them on disjoint validation families; it changes how evaluators are built ($\phi$), not how tasks are solved ($\pi$). We do not evaluate the meta loop: \cref{sec:experiments} runs the inner loop only.

\subsection{Self-attack and regression-gated repair}
\label{sec:repair}
A drafted check may be implemented incorrectly, or an obligation may be missing or mis-stated. QC (Adversarial Testing and Repair \& Validate in \cref{fig:framework}) targets both. At version $v$ it constructs labeled cases of the form defined in \cref{sec:setup},
\begin{equation}
 A_v=\{(y_j,i_j,z^\star_j)\}_j\sim\Attack(R_v,\Info),
 \label{eq:attack}
\end{equation}
of three kinds: \emph{control} cases should pass, \emph{error} cases violate the targeted criterion (e.g., an edited cell that breaks the completion rate), and \emph{boundary} cases are near misses. Reference judges, which are language models, label anonymized cases against public evidence, blind to the RewardPack, the attack rationale, and the predicted label. The cases join the repair ledger $Q_v=Q_{v-1}\cup A_v$, with $Q_{-1}=\varnothing$. Agreement, disagreement, and unresolved judgments are as defined in \cref{sec:setup}; judgments are recorded once per version and replayed. Let $G_v\subseteq Q_v$ be the cases on which $R_v$ agrees and $F_v\subseteq A_v$ the new cases on which it disagrees. A case whose repair is rejected stays in the ledger as a disagreement and is not targeted again, although a later accepted repair that also fixes it adds it to the agreed cases.

Jury revises missing or mis-stated obligations and Builder repairs checks, producing a \emph{minimal} candidate $\widetilde R\sim\Repair_\phi(R_v,F_v)$ that preserves unrelated criteria and checks and remeasures only affected bindings. Writing $\widetilde G$ and $\widetilde U$ for the cases of $Q_v$ on which the candidate agrees and is unresolved, it is accepted only if
\begin{equation}
 \begin{aligned}
 &\text{no regression:} && G_v\subseteq\widetilde G,\\
 &\text{fixes a failure:} && F_v\cap\widetilde G\neq\varnothing,\\
 &\text{no unresolved comparison:} && \widetilde U=\varnothing.
 \end{aligned}
 \label{eq:accept}
\end{equation}
The implementation also requires unchanged known reference labels and supported fresh boundary cases around the repair (\cref{app:runtime_details}). The last condition of \cref{eq:accept} is strict: a ledger case that stays unresolved under every version blocks later repairs until its evidence is resolved, and QC then keeps the incumbent. Then $R_{v+1}=\widetilde R$ if the candidate is accepted and $R_{v+1}=R_v$ otherwise, and QC continues until an attack finds no failure or its budget is exhausted.

The acceptance rule makes agreement monotone, $G_v\subseteq G_{v+1}$: a reference case the reward once matched is never lost (\cref{prop:ratchet} in \cref{app:ratchet_proof}). The guarantee has three limits. It is relative to the reference labels, which are themselves fallible model judgments whose accuracy we do not measure, so a systematically wrong label is kept as firmly as a correct one. It covers only accumulated cases; accuracy on unseen artifacts requires held-out evaluation. And it covers only QC repairs: at a round boundary (stage 3), BoundaryUpdate in \cref{alg:main} keeps, rebuilds, or revises the reward under the caller's rules, which need not preserve $G_v$ (\cref{app:runtime_details}).

\subsection{Using a frozen reward}
\label{sec:use}
After QC the latest version is frozen for execution (stage 2 of \cref{fig:framework}); we write $R$ for it below. Products, receipts, and defects found during use inform construction only at the next round boundary (stage 3).

\paragraph{Critic feedback.} Actor steps append actions and observations to the context, and the terminal review of a reward-guided delivery $y_T$ appends Critic feedback computed from targeted checks of the frozen reward:
\begin{equation}
 \begin{aligned}
 &\text{actor step:} && a_t\sim\pi(\cdot\mid h_t), && h_{t+1}=h_t\oplus(a_t,o_{t+1}),\\
 &\text{terminal review:} && f\sim\Critic_R(\cdot\mid h_T,R_{S}(y_T)), && h'_T=h_T\oplus f,
 \end{aligned}
 \label{eq:feedback}
\end{equation}
where $R_{S}(y_T)$ restricts \cref{eq:measurement} to the set $S$ of checks the review targets (at most three in our runs). Critic suggests, requests clarification, or abstains, citing supporting receipts. When the review is nonempty and budget remains, the same actor session continues from $h'_T$, and the revised delivery is then rechecked. This loop uses the limits and stopping conditions in \cref{app:evaluated_config}, with $\pi$ and $R$ fixed throughout. Because unselected checks stay unverified, stopping does not certify that the whole RewardPack passes.

\paragraph{Candidates and selection.} Each round generates three candidates: a \emph{task-only candidate} $y^\circ$ from the task alone, a \emph{reward-guided revision} using the RewardPack and terminal Critic review, and a \emph{hypothesis-guided candidate} whose actor receives guidance from a Hypothesis module but neither the RewardPack nor the other candidates. The complete candidate-generation protocol is specified in \cref{app:evaluated_config}. In this second Jury role, three members rank all candidates generated so far from their artifacts and receipts, and a reducer adjudicates the rankings against the evidence, giving an order $\succ$. The selector chooses among the current round's candidates $\mathcal Y_n$ and keeps $y^\circ$ unless one ranks strictly above it without a known regression:
\begin{equation}
 y^\star=\max\nolimits_{\succ}\bigl(\{y^\circ\}\cup\mathcal Y^{+}\bigr),\qquad
 \mathcal Y^{+}=\{y\in\mathcal Y_n:\ y\succ y^\circ,\ \text{no test of $R$ that $y^\circ$ passes fails on $y$}\}.
 \label{eq:select}
\end{equation}
We call $y^\star$ \emph{Selected by reward}; ties are broken as in \cref{app:evaluated_config}. The selector sees no benchmark scores or grading receipts. The reward enters selection as evidence for the Jury and as a test-level veto, not as a scalar score; the order agreement in \cref{sec:main_results} evaluates the reward's own pass-share ordering (\cref{sec:setup}), not the Jury's ranking. Each round therefore makes many more model calls than a single actor run, across the Jury, Builder, QC, Critic, and ranking roles; \cref{app:existing_costs} explains why we report no total-cost or equal-budget comparison.

\begin{algorithm}[t]
\caption{Reward Engine: inner and meta loops. Experiments evaluate the inner loop only.}
\label{alg:main}
\small\raggedright
\begin{algorithmic}[1]
\Require task families $\mathcal D_{\rm tr},\mathcal D_{\rm val},\mathcal D_{\rm te}$, each task $x$ with initial task-only runs $\mathcal B$; construction state $\phi_0=(\omega_0,\Memory_0)$; construction, attack, and execution budgets
\Procedure{InnerLoop}{$x,\mathcal B;\phi$} \Comment{public information $\Info$ includes $x$ and $\mathcal B$}
  \For{round $n=1,\dots,N$}
    \AlgStage{2}{Stage 1: build reward}
    \State \textbf{if} $n=1$: $R\gets\Build_\phi(\Info)$;\ $Q\gets\varnothing$ \Comment{\cref{eq:build_operator}}
    \State \textbf{else}: $R\gets\Call{BoundaryUpdate}{R,Q,\Delta\Info}$
    \Repeat \Comment{QC: self-attack and repair}
      \State $A\gets\Attack(R,\Info)$, blindly labeled;\ $Q\gets Q\cup A$;\ $F\gets$ cases of $A$ on which $R$ disagrees
      \State \textbf{if} $F\neq\varnothing$: $\widetilde R\gets\Repair_\phi(R,F)$;\ $R\gets\widetilde R$ if \cref{eq:accept} and the checks of \cref{app:runtime_details} hold
    \Until{$F=\varnothing$ or the attack budget is spent}
    \AlgStage{2}{Stage 2: execute by reward ($R$ frozen)}
    \For{$h_0\in\Call{CandidateContexts}{x,R}$} \Comment{candidate protocol in \cref{app:evaluated_config}}
      \State $h\gets h_0$
      \While{delivery incomplete and the execution budget remains}
        \State $a\sim\pi(\cdot\mid h)$; observe $o$;\ $h\gets h\oplus(a,o)$
      \EndWhile
      \State \textbf{if} reward-guided: $h\gets$ \Call{TerminalReviewLoop}{$h,R$} \Comment{\cref{eq:feedback}}
    \EndFor
    \State Jury ranks all candidates so far from artifacts and receipts;\ select $y^\star$ among this round's \Comment{\cref{eq:select}}
    \AlgStage{2}{Stage 3: update reward}
    \State $\Delta\Info\gets$ products, receipts, and use-time defects;\ $\Memory\gets\Memory\cup{}$construction lessons
  \EndFor
  \State \Return $y^\star$
\EndProcedure
\AlgStage{0}{Meta loop: update the construction state}
\For{$k=0,1,\dots$ until the incumbent is retained}
  \State $\Phi_k\gets\Propose(\phi_k;\mathcal D_{\rm tr})\cup\{\phi_k\}$ \Comment{from \textsc{InnerLoop} runs on $\mathcal D_{\rm tr}$}
  \State $\phi_{k+1}\gets\Select_{\mathcal D_{\rm val}}(\Phi_k)$
\EndFor
\State \Return \Call{InnerLoop}{$x,\mathcal B;\phi_k$} for each $x\in\mathcal D_{\rm te}$ \Comment{Strategy $\omega_k$ fixed; Memory still grows}
\end{algorithmic}
\end{algorithm}

\section{Experiments}
\label{sec:experiments}
We ask three questions in \cref{sec:main_results}: does Reward Engine improve the delivered artifact; does the frozen reward select better candidates than reward-free policies; and does its ordering agree with the benchmark grader's? \Cref{app:sensitivity} reports sensitivity analyses.

\subsection{Experimental setup}
\label{sec:exp_setup}
\paragraph{Benchmarks.} \emph{WorkBuddy Bench Office} has 50 mixed-format office tasks scored 0--100 by deterministic rules and a model judge \citep{workbuddybench2026}. \emph{Claw-Eval} provides executable environments and human-checked rubrics \citep{ye2026claweval,clawevaldocs2026}; we use 104 frozen task versions from 68 scenarios (36 bilingual pairs), scored 0--100 by its native grader with a custom judge (\cref{app:benchmarks}).

\paragraph{Compared outputs.} Each task is executed in sequential rounds. In every round, the system generates a fresh task-only candidate from the public task alone, a reward-guided revision that also receives the RewardPack and terminal Critic review, and a hypothesis-guided candidate; the selector then returns Selected by reward (\cref{sec:use}). We compare Selected by reward with Baseline, Revision, uniform choice among the generated candidates, and an Oracle that selects the candidate with the highest benchmark score. Baseline and policy aggregation are defined in \cref{app:baselines}, and the complete candidate-generation protocol is given in \cref{app:evaluated_config}.

\paragraph{Metrics.} We report mean score, average pass rate, three-run success coverage ($\mathrm{pass@3}$), and three-run consistency ($\mathrm{Pass}^{3}$), together with paired 95\% confidence intervals that resample WorkBuddy Bench Office tasks or Claw-Eval scenarios. For selection, we additionally report regret to the oracle and the share of the uniform-to-oracle gap recovered. Metric definitions, aggregation, and statistical procedures are specified in \cref{app:baselines,app:statistics}; per-round performance is reported in \cref{app:per_round}.

\subsection{Main results}
\label{sec:main_results}\label{sec:selection}\label{sec:order_agreement}\label{sec:measurement}
\begin{table}[t]
\caption{Delivery quality. Brackets show paired 95\% confidence intervals for the mean difference from Baseline; Sign $p$ is a two-sided sign test on per-task differences. Bold marks the best mean and pass rate. Metric definitions are in \cref{app:statistics}.}
\label{tab:main_results}\label{tab:benchmark_main}\label{tab:benchmark_wtl}\label{tab:main_office50}\label{tab:main_claw}
\centering\small\setlength{\tabcolsep}{3pt}
\begin{tabular}{lrrlr}
\toprule
Method & Mean & Avg. pass (\%) & $\Delta$ mean [95\% CI] & Sign $p$ \\
\midrule
\multicolumn{5}{l}{\textit{WorkBuddy Bench Office}} \\
Baseline & 81.05 & 66.67 & --- & --- \\
Revision & 81.10 & 70.67 & +0.05 [$-$2.42, 2.27] & 0.0649 \\
Selected by reward & \textbf{82.21} & \textbf{71.33} & +1.16 [$-$1.60, 3.96] & 0.0854 \\
\midrule
\multicolumn{5}{l}{\textit{Claw-Eval}} \\
Baseline & 89.44 & 89.74 & --- & --- \\
Revision & 90.97 & \textbf{93.59} & +1.53 [$-$0.42, 3.33] & $<10^{-4}$ \\
Selected by reward & \textbf{91.47} & \textbf{93.59} & +2.03 [0.90, 3.21] & 0.0008 \\
\bottomrule
\end{tabular}
\end{table}

\begin{table}[t]
\caption{Three-run success coverage ($\mathrm{pass@3}$) and consistency ($\mathrm{Pass}^{3}$). Bold marks the best value in each column. Definitions and aggregation are given in \cref{app:baselines,app:statistics}.}
\label{tab:three_round_success}
\centering\small
\begin{tabular}{lrrrr}
\toprule
 & \multicolumn{2}{c}{WorkBuddy Bench Office} & \multicolumn{2}{c}{Claw-Eval} \\
Method & $\mathrm{pass@3}$ (\%) & $\mathrm{Pass}^{3}$ (\%) & $\mathrm{pass@3}$ (\%) & $\mathrm{Pass}^{3}$ (\%) \\
\midrule
Baseline & 85.00 & 45.00 & 98.08 & 78.37 \\
Revision & 84.00 & \textbf{56.00} & 98.08 & \textbf{87.50} \\
Selected by reward & \textbf{86.00} & 54.00 & \textbf{100.00} & 83.65 \\
\bottomrule
\end{tabular}
\end{table}

\begin{figure}[t]
\centering
\includegraphics[width=\textwidth]{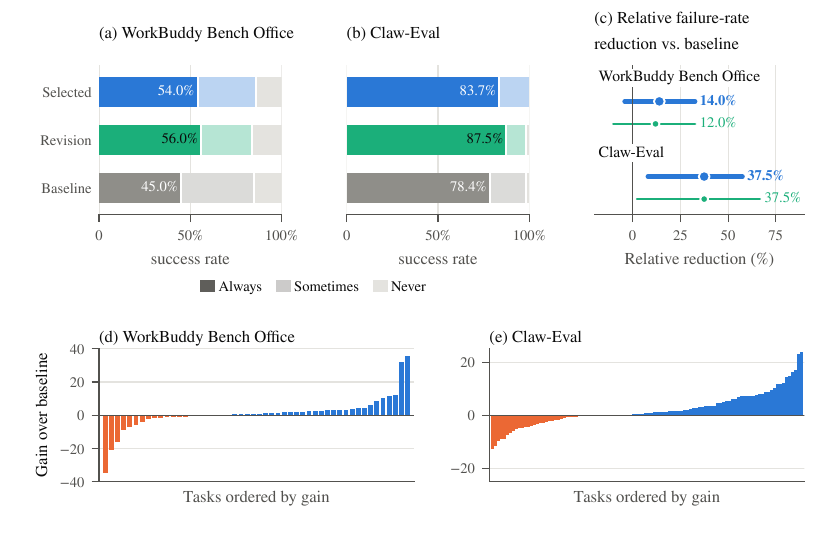}
\caption{Benchmark results: (a, b) success consistency, where Always, Sometimes, and Never denote three, one or two, and zero successes within a three-run group; dark segments report $\mathrm{Pass}^{3}$. (c) Relative failure reduction with 95\% confidence intervals. (d, e) Sorted per-task score gains of Selected by reward over Baseline, retaining all tasks. Definitions and statistical details are given in \cref{app:baselines,app:statistics}.}
\label{fig:benchmark_results}\label{fig:main_results}
\end{figure}

\paragraph{Selected by reward has the highest mean score and coverage on both benchmarks.} It has the highest mean score on WorkBuddy Bench Office and Claw-Eval (\cref{tab:main_results}), gaining $+2.03$ ($[0.90,3.21]$) on Claw-Eval and $+1.16$ ($[-1.60,3.96]$) on WorkBuddy Bench Office, where the interval includes zero. It also has the highest success coverage, reaching every Claw-Eval task at least once ($\mathrm{pass@3}=100\%$), and raises $\mathrm{Pass}^{3}$ from 45.00\% to 54.00\% and from 78.37\% to 83.65\% (\cref{tab:three_round_success}; \cref{fig:benchmark_results}a,b).

\paragraph{Failures fall most on Claw-Eval.} Relative to Baseline, the run failure rate falls by 37.5\% on Claw-Eval, with an interval excluding zero, and by 14.0\% on WorkBuddy Bench Office, with an interval including zero (\cref{fig:benchmark_results}c). The per-task distributions retain both gains and losses (\cref{fig:benchmark_results}d,e). Always delivering the revision attains the highest $\mathrm{Pass}^{3}$ (56.00\% and 87.50\%) but a lower mean and coverage, so selection trades some consistency for higher average quality.

\begin{table}[t]
\caption{Selection utility of the frozen reward on the generated candidate pools. Gain is Selected minus uniform choice; headroom is the share of the uniform-to-oracle gap recovered; regret is oracle minus Selected. Policy comparisons are in \cref{tab:selection_quality}.}
\label{tab:reward_utility_summary}
\centering\small\setlength{\tabcolsep}{4pt}
\begin{tabular}{lrr}
\toprule
Quantity & WorkBuddy Bench Office & Claw-Eval \\
\midrule
Uniform $\rightarrow$ selected score & 80.38 $\rightarrow$ 82.21 & 89.11 $\rightarrow$ 91.47 \\
Gain over uniform [95\% CI] & +1.83 [0.22, 3.61] & +2.36 [1.56, 3.20] \\
Headroom recovered (\%) & 25.76 & 39.40 \\
Selection regret [95\% CI] & 5.29 [2.92, 8.53] & 3.63 [2.81, 4.49] \\
Best-hit rate (\%) & 49.33 & 58.65 \\
\bottomrule
\end{tabular}
\end{table}

\paragraph{Does the frozen reward select better candidates?} All policies choose among the same generated candidates, so selection quality can be compared on equal footing. Selecting by the frozen reward gains $+1.83$ ($[0.22,3.61]$) over uniform choice on WorkBuddy Bench Office and $+2.36$ ($[1.56,3.20]$) on Claw-Eval, recovering 25.76\% and 39.40\% of the uniform-to-oracle gap (\cref{tab:reward_utility_summary}). Its mean regret, 5.29 and 3.63, is the lowest among the non-oracle policies, compared with 6.40 and 4.13 for always revising.

\begin{table}[tb]
\caption{Order agreement and $\tau_+$ with 95\% confidence intervals.}
\label{tab:order_agreement}
\centering\small
\begin{tabular}{lrr}
\toprule
Task set & Agreement (\%) [95\% CI] & $\tau_+$ [95\% CI] \\
\midrule
\multicolumn{3}{l}{\textit{WorkBuddy Bench Office}} \\
All tasks & 82.33 [74.67, 89.33] & 0.647 [0.493, 0.787] \\
Distinguishable tasks & 84.69 [77.21, 91.50] & 0.694 [0.544, 0.830] \\
\midrule
\multicolumn{3}{l}{\textit{Claw-Eval}} \\
All tasks & 80.53 [74.75, 86.08] & 0.611 [0.495, 0.722] \\
Distinguishable tasks & 81.72 [75.96, 87.21] & 0.634 [0.519, 0.744] \\
\bottomrule
\end{tabular}
\end{table}

\paragraph{Does the reward agree with the benchmark grader?} We re-evaluate candidates against frozen RewardPacks and compare their within-task orderings with the benchmark grader. Order agreement, the share of grader-ordered candidate pairs that the reward does not reverse, is 82.33\% on WorkBuddy Bench Office and 80.53\% on Claw-Eval, and 84.69\% and 81.72\% on distinguishable tasks (\cref{tab:order_agreement}). The packs may have seen evaluated candidates; this analysis does not measure generalization to unseen candidates.

\section{Related work}
\label{sec:related}
\paragraph{Judges, rubrics, and verifiers.} Model judges scale open-ended evaluation but are sensitive to position and presentation \citep{zheng2023judging,wang2024fair}. Fine-grained rubric judges \citep{kim2024prometheus}, rubrics and checklists as rewards \citep{gunjal2026rubrics,viswanathan2025checklists}, reusable rubric evaluation \citep{rao2026autorubric}, and natural-language unit tests \citep{saadfalcon2025lmunit} make criteria explicit, and Agent-as-a-Judge adds agentic inspection \citep{zhuge2025agentjudge}. Trained verifiers and process reward models score solutions or steps \citep{cobbe2021verifiers,lightman2024verify,zhang2025genverifier}, and AgentRewardBench documents disagreement between automatic and expert judgments of web-agent trajectories \citep{lu2025agentrewardbench}. Our unit of evaluation is instead a criterion bound to an executed procedure whose receipt names the artifact, scope, and version, with explicit status for errors and unresolved evidence.

\paragraph{Constructing and repairing rewards.} Language models can design rewards from descriptions \citep{kwon2023reward} and write reward code refined with training feedback \citep{ma2024eureka,xie2024text2reward}; CodeT selects programs with generated tests \citep{chen2023codet}; and EvoRubrics co-evolves rubrics adversarially with the policy \citep{ding2026evorubrics}. Reward Engine instead tests the reward against counterexamples before use, gates repairs by regression on accumulated cases as in software regression testing, and freezes the result while it is used (\cref{tab:related_comparison}). Freezing fixes the standard but does not remove proxy misspecification, since a fixed proxy can still be overoptimized or hacked \citep{skalse2022hacking,gao2023overoptimization,pan2024incontext}.

\paragraph{Feedback, refinement, and persistent experience.} Self-Refine, Reflexion, CRITIC, and TextGrad revise outputs or system components with language feedback \citep{madaan2023selfrefine,shinn2023reflexion,gou2024critic,yuksekgonul2025textgrad}, although re-prompting alone can account for gains attributed to critique \citep{stechly2025verification}. ExpeL, ADAS, and SkillRL retain lessons, agent designs, or skills \citep{zhao2024expel,hu2025adas,xia2026skillrl}, and Constitutional AI, PPO, and GRPO update parameters from feedback or rewards \citep{bai2022constitutional,schulman2017ppo,shao2024deepseekmath}. Reward Engine exposes the measurement behind its feedback, and its persistent target is evaluator construction rather than task solving.
\label{sec:related_end}

\section{Conclusion}
\label{sec:conclusion}
Reward Engine makes reward construction part of solving open-ended tool tasks. From public task requirements and execution experience, it builds evidence-grounded checks for required correctness and graded quality, tests them against counterexamples, and holds the resulting reward fixed while it guides revision and selection. Updates between rounds incorporate new evidence without changing the standard during use. On WorkBuddy Bench Office and Claw-Eval, this approach raises three-run consistency over Baseline and improves selection over uniform choice among the same candidates by 1.83 and 2.36 points, respectively, with 95\% confidence intervals excluding zero. Mean delivery gains over Baseline are clearest on Claw-Eval. The constructed reward also does not reverse the benchmark graders' ordering on over 80\% of evaluated candidate pairs. These results support constructing and validating rewards as a practical route to improving tool-agent outputs at inference time. Applying the same reward interface to reinforcement learning is a direction for future work.

\section*{Code Availability}
The Reward Engine implementation is available at \url{https://github.com/LeibnizFlow/reward-engine}. It supports OpenCode and Jiuwen Swarm as execution runtimes.

\bibliography{references}
\bibliographystyle{plainnat}

\clearpage
\appendix
\etocdepthtag.toc{mtappendix}
\etocsettagdepth{mtmain}{none}
\etocsettagdepth{mtappendix}{subsection}
\etocsettocstyle{\section*{Appendix contents}\begingroup\small\setlength{\parskip}{1pt}}{\endgroup}
\etocsetstyle{section}{}{}{\par\vspace{4pt}\noindent\textbf{\makebox[1.4em][l]{\etocnumber}\etocname}\hfill\textbf{\etocpage}\par}{}
\etocsetstyle{subsection}{}{}{\noindent\hspace*{1.4em}\makebox[2.2em][l]{\etocnumber}\etocname\leaders\hbox to .6em{\hss.\hss}\hfill\etocpage\par}{}
\tableofcontents

\noindent\Cref{app:related_work} compares Reward Engine with related methods, \cref{app:method_details} gives the notation, RewardPack schema, and implementation details, and \cref{app:theory} proves the regression-gate invariant. \Cref{app:benchmarks} describes the benchmarks, baselines, implementation, statistical procedures, grading provenance, and recorded costs. \Cref{app:additional_results} supports each main-text result in the order of \cref{sec:experiments}: delivery quality, per-round results, sensitivity analyses, selection, agreement with the benchmark grader, and auxiliary measurement case studies.

In method-comparison tables, bold marks the best displayed point estimate among comparable methods, including ties. Higher scores and success rates, and lower regret and error rates, are preferred. Oracle upper bounds and label-defining comparators are excluded; bold does not imply statistical significance.

\section{Comparison with related methods}
\label{app:related_work}
\Cref{tab:related_comparison} compares Reward Engine with the evaluators and reward-construction methods discussed in \cref{sec:related} on the properties that define its reward. Marks describe each method as specified and evaluated in the cited work, not extensions it could support; feedback and experience methods, which revise the actor rather than the evaluator, are discussed only in \cref{sec:related}. Most methods already keep the evaluator fixed while it is used. What distinguishes Reward Engine is the combination of evidence receipts, per-task counterexample testing, and a regression gate. These make evaluator failures visible and keep repairs from losing earlier agreement, but they do not by themselves make the reward correct: the gate holds only on generated cases with fallible reference labels (\cref{prop:ratchet}).

\begin{table}[htbp]
\caption{Properties of evaluators and reward-construction methods as evaluated in the cited work. Yes: provided; Part.: partially provided (see notes); No: not provided; ---: not applicable, because the work only evaluates and nothing is optimized against the evaluator.}
\label{tab:related_comparison}
\centering\footnotesize\setlength{\tabcolsep}{3pt}
\begin{tabular}{@{}>{\raggedright\arraybackslash}p{.25\linewidth}*{5}{>{\centering\arraybackslash}p{.1\linewidth}}>{\raggedright\arraybackslash}p{.14\linewidth}@{}}
\toprule
Method & Executable checks$^a$ & Evidence receipts$^b$ & Counter\-example testing$^c$ & Regression gate$^d$ & Frozen during use$^e$ & Reported use \\
\midrule
\multicolumn{7}{@{}l}{\textit{Judges, rubrics, and verifiers}} \\
LLM-as-a-judge \citep{zheng2023judging} & No & No & No & No & --- & Evaluation \\
Prometheus \citep{kim2024prometheus} & No & No & No & No & --- & Evaluation \\
LMUnit \citep{saadfalcon2025lmunit} & No & No & No & No & --- & Evaluation \\
Agent-as-a-Judge \citep{zhuge2025agentjudge} & No & Part. & No & No & --- & Evaluation \\
Autorubric \citep{rao2026autorubric} & No & Part. & No & No & Yes & Evaluation, RL \\
Rubrics as Rewards \citep{gunjal2026rubrics} & No & No & No & No & Yes & RL \\
Checklist feedback \citep{viswanathan2025checklists} & Part. & No & No & No & Yes & Preference optimization \\
Trained verifiers \citep{cobbe2021verifiers,lightman2024verify,zhang2025genverifier} & No & No & No & No & Yes & Reranking \\
\midrule
\multicolumn{7}{@{}l}{\textit{Constructing and repairing rewards}} \\
LM reward design \citep{kwon2023reward} & No & No & No & No & Yes & RL \\
Eureka \citep{ma2024eureka} & Yes & No & No & No & Yes & RL \\
Text2Reward \citep{xie2024text2reward} & Yes & No & No & No & Yes & RL \\
CodeT \citep{chen2023codet} & Yes & No & No & No & Yes & Reranking \\
EvoRubrics \citep{ding2026evorubrics} & No & No & Part. & No & No & RL \\
\midrule
Reward Engine (ours) & Yes & Yes & Yes & Yes & Yes & Selection, revision \\
\bottomrule
\end{tabular}

\smallskip
\begin{minipage}{\linewidth}\footnotesize\raggedright
$^a$Each criterion is bound to a stored, re-runnable procedure (code or a tool recipe). Reward Engine's procedures include model judgments over located evidence, so executability does not imply a deterministic oracle. Checklist feedback writes verifier programs only for items it can check exactly and averages them with judge scores.\\
$^b$Each judgment retains the inspected evidence and separates unresolved or failed measurement from a violation. Autorubric records a per-criterion explanation and an explicit cannot-assess verdict; Agent-as-a-Judge gathers workspace evidence for each requirement.\\
$^c$The evaluator for a task is tested on labeled counterexamples before use. Several works measure agreement with human labels in aggregate \citep{zheng2023judging,zhuge2025agentjudge,rao2026autorubric,lu2025agentrewardbench}; EvoRubrics adapts its rubric generator adversarially during training rather than testing it on labeled cases.\\
$^d$A revision is accepted only if no previously agreed case is lost; for Reward Engine this holds on generated cases with fallible reference labels (\cref{prop:ratchet}).\\
$^e$The evaluator is fixed while a policy is trained, reranked, or revised against it, that is, during optimization. Eureka and Text2Reward fix the reward within each training run and revise it between runs, as Reward Engine does between rounds; EvoRubrics updates its rubrics at every training step.
\end{minipage}
\end{table}
\FloatBarrier

\section{Method details}
\label{app:method_details}
This section gives the notation, the RewardPack schema, and the implementation details behind the construction, repair, and use stages summarized in \cref{sec:method}. \Cref{fig:framework} shows the evaluated configuration; Builder appears there as Reward Builder and Harness Builder, and QC as Adversarial Testing and Repair \& Validate.

\subsection{Notation}
\label{app:notation}

\Cref{tab:notation} collects the notation used in \cref{sec:setup,sec:method} and this appendix. The model has three layers: the actor $\pi$, the reward $R$, and the construction state $\phi$ that builds rewards. Actor steps $t$, reward versions $v$, and experimental rounds $n$ are distinct indices: additional actor steps need not change the reward.

\begingroup
\small
\renewcommand{\arraystretch}{1.12}
\setlength{\tabcolsep}{5pt}
\begin{longtable}{@{}>{\raggedright\arraybackslash}p{.25\linewidth}>{\raggedright\arraybackslash}p{.71\linewidth}@{}}
\caption{Notation for reward construction, execution, and revision.}\label{tab:notation}\\
\toprule
Symbol & Meaning \\
\midrule
\endfirsthead
\multicolumn{2}{l}{\small\textit{Notation (continued)}}\\
\toprule
Symbol & Meaning \\
\midrule
\endhead
\bottomrule
\endfoot
\multicolumn{2}{@{}l}{\textit{Task and actor}} \\
$x$ & Public task requirements. \\
$\Info$, $\Delta\Info$ & Public information: the task, genuine user updates, permitted public knowledge, the execution history, and $\mathcal B$; it excludes hidden benchmark answers and grading outputs. $\Delta\Info$ is the new public evidence admitted at a round boundary. \\
$\mathcal B$ & Initial task-only runs available for construction (three per task). \\
$\pi$ & Actor: a model run in the Actor Harness, held fixed. \\
$h_t$, $a_t$, $o_t$ & Actor context, action, and observation at step $t$. \\
$y$ & Product: delivered artifact with its interaction trace and environment state. \\
\midrule
\multicolumn{2}{@{}l}{\textit{Reward}} \\
$R=\{(c_i,\kappa_i)\}_{i=1}^{m}$ & Reward: criteria $c_i$ bound to executable checks $\kappa_i$; $R_v$ is version $v$, stored in a RewardPack. \\
$\kappa_i(y;x)=(z_i,e_i)$ & Result of check $\kappa_i$ on product $y$ and task inputs $x$; $R(y)=(\kappa_i(y;x))_{i=1}^{m}$ (\cref{eq:measurement}). \\
$z_i$, $\mathcal Z_i$, $\bot$ & Judgment, its value space ($\{\mathsf{pass},\mathsf{fail}\}$ for a testcase, $\{\mathsf{fail},\mathsf{partial},\mathsf{full}\}$ for a rubric), and no judgment. Multiple bound tests give the tuple in \cref{app:formal_contracts}. \\
$e_i$ & Receipt: evidence references and execution status; an execution error is distinct from an observed criterion failure. \\
\midrule
\multicolumn{2}{@{}l}{\textit{Construction and repair}} \\
$\phi=(\omega,\Memory)$ & Construction state: Strategy and Memory; $\phi_k$ after $k$ meta-loop updates. \\
$\Build_\phi$, $\Attack$, $\Repair_\phi$ & Construction, self-attack, and repair operators (\cref{eq:build_operator,eq:attack}, \cref{sec:repair}). \\
$(y_j,i_j,z_j^\star)$ & Labeled case $j$: product, targeted criterion, and reference label. \\
$A_v$, $Q_v$ & Cases generated at version $v$, and the repair ledger $Q_v=Q_{v-1}\cup A_v$. \\
$G_v$, $F_v$ & Cases of $Q_v$ on which $R_v$ agrees with the reference label; cases of $A_v$ on which it disagrees. \\
$\widetilde R$; $\widetilde G$, $\widetilde U$ & Candidate repair; cases of $Q_v$ on which it agrees or is unresolved (\cref{eq:accept}). \\
\midrule
\multicolumn{2}{@{}l}{\textit{Use and meta loop}} \\
$\Critic_R$, $f$ & Critic under the frozen reward $R$, and the feedback it appends at a terminal review (\cref{eq:feedback}). \\
$S$, $R_S$ & Checks targeted by a terminal review, and the restriction of $R$ to them (\cref{eq:feedback}). \\
$\mathcal Y_n$, $y^\circ$, $y^\star$, $\succ$ & Candidates of round $n$, task-only candidate, Selected by reward, and the reduced Jury ranking (\cref{eq:select}). \\
$\mathcal D_{\rm tr}$, $\mathcal D_{\rm val}$, $\mathcal D_{\rm te}$ & Training, validation, and test task families of the meta loop. \\
$t$, $v$, $k$, $n$ & Actor steps, reward versions, meta-loop updates, and rounds; $i$, $j$, and $\ell$ index criteria, labeled cases, and the tests bound to a criterion. \\
\end{longtable}
\endgroup

Budgets, scopes, and delivery contracts are named in words rather than symbols. Benchmark scores and success statistics use separate notation, defined in \cref{app:statistics}.

\subsection{RewardPack schema and judgment semantics}
\label{app:formal_contracts}
A RewardPack stores one reward version $R_v=\{(c_i,\kappa_i)\}_{i=1}^{m}$ together with its scope, version metadata, and repair ledger; \cref{tab:rewardpack_schema} lists the stored records.

\begin{table}[htbp]
\caption{Records stored in a RewardPack. Every published criterion has at least one bound test; obligations that are deferred or cannot be mapped remain unknown in a separate coverage ledger.}
\label{tab:rewardpack_schema}
\centering\footnotesize
\begin{tabular}{@{}>{\raggedright\arraybackslash}p{.2\linewidth}>{\raggedright\arraybackslash}p{.1\linewidth}>{\raggedright\arraybackslash}p{.64\linewidth}@{}}
\toprule
Record & Producer & Fields \\
\midrule
Criterion ($c_i$) & Jury & Identifier; standard; hard or soft kind with its origin and evidence; applicability and priority rationale; measurement question; subrequirements; fail, partial, and full anchors. \\
Test binding ($\kappa_i$) & Builder & Test and criterion identifiers; checker key, which a registry dispatches to a function or model evaluator; arguments; question; rubric; optional tool recipe of tool calls and restricted computations. \\
Measurement record (receipt $e_i$) & Harness & Artifact and environment snapshots; scope; reward and check versions; input references; procedure type; observations; judgment; execution status. \\
Scope & Caller & Applicable requirements and conversation, admitted sources, user amendments, and environment assumptions. \\
Version metadata & QC & Version identity and hash; parent version; changed criteria and checks; fresh measurements; review records bound to the identity, scope, and hash; selection rationale. \\
Repair ledger ($Q_v$) & QC & Case identity; material (artifact and history); case type (control, error, or boundary); targeted criterion; reference label; agreement recorded once per version and replayed. \\
\bottomrule
\end{tabular}
\end{table}

\paragraph{Judgment semantics.} For criterion $i$ with bound tests $\ell\in L_i$, the judgment is the tuple $z_i=(z_{i\ell})_{\ell\in L_i}$ with $z_{i\ell}\in\mathcal Z_{i\ell}\cup\{\bot\}$ binary or graded; the single-test judgments of \cref{sec:setup} are the case $|L_i|=1$. Three outcomes are kept apart throughout: a failing judgment records an observed violation, $\bot$ records insufficient evidence, and a parser or tool exception is an execution error, recorded in the receipt $e_i$, that establishes neither a violation nor its cause. No scalar is imposed. A consumer that needs one supplies a reduction rule, under which $(\mathsf{pass},\bot)$ earns no implicit partial credit and an all-unknown tuple remains unresolved; hard or soft kinds and priorities imply no global weights, vetoes, or aggregation rule. A receipt may be reused only if its artifact, environment snapshot, requirement scope, and check version all match the current measurement. This is admissibility for reuse, not semantic truth, and it cannot supply missing observations, so a modified workbook needs fresh receipts and a checker repair keeps only receipts whose bindings still match.

\paragraph{Information boundary.} Construction and feedback use only $\Info$ (\cref{sec:setup}). Genuine user updates enter $\Info$, and internal suggestions never add requirements. Hidden answers and grader outputs are excluded by design. Public benchmark materials that an actor retrieves count as public knowledge (\cref{tab:claw_eval_three_round_sensitivity}).

\subsection{Construction and regression-gated repair}
\label{app:runtime_details}
\paragraph{Jury.} Three Jury members draft obligations from the public requirements and permitted knowledge, then reconsider them against the initial task-only artifacts. Each phase needs at least two usable drafts; missing outputs do not count as agreement, and too few drafts or a failed reduction yield no publishable specification. A reducer merges both phases, and validation maps, defers, or excludes with a reason every identified obligation. Obligations that were never proposed remain a semantic omission risk that validation cannot detect. Knowledge, strategies, and QC hypotheses may suggest checks but do not establish requirements.

\paragraph{Builder and delivery contracts.} Builder produces the bindings in \cref{tab:rewardpack_schema} under the caller's scope, construction budget, and delivery contract, returning an admissible RewardPack or an explicit construction failure that remains in the delivery denominator. The Outcome contract requires schema, binding, and smoke checks and permits synthetic fixtures. Publication additionally requires a receipt for every bound test on a real artifact, not all of them errors, a completed review, and a selection rationale. The artifact must come from a sandboxed actor run charged to construction or a verified archive captured before grading.

\paragraph{Self-attack and repair.} Adversarial review executes the control, error, and boundary cases of \cref{sec:repair}. Reference models label anonymized material against public task evidence, blind to the RewardPack, the attack rationale, and the predicted labels, and replay preserves case identities and material. Beyond \cref{eq:accept}, the implementation accepts a candidate only if known reference labels are unchanged and fresh boundary cases around the repair are supported; otherwise the incumbent is kept. Structural review checks bindings but not semantics, so a rebuild without adversarial review leaves its semantic correction unverified. QC routes confirmed checker defects to Builder, missing or mis-stated obligations to Jury, and independently established artifact defects to actor feedback, and leaves uncertain attributions unresolved; infrastructure failures permit bounded remeasurement. A checker repair preserves the criteria and unrelated checks, and a criterion repair remeasures only the affected bindings. Defects found while a reward is in use are queued and never modify the active version. The BoundaryUpdate operation in \cref{alg:main} keeps, rebuilds, or revises the reward at a round boundary under comparison and tie rules fixed by the caller, carrying forward the repair ledger. It records the resulting reward as the next version $R_{v+1}$; retaining the incumbent leaves the reward contents unchanged. These boundary transitions follow the caller's rules rather than \cref{eq:accept}. After QC, the latest version is frozen for the execution segment.

\subsection{Execution with a frozen reward}
\label{sec:critic}
\paragraph{Segments and Critic.} An execution segment starts from a caller-supplied context $h_0$ and ends at completion, budget exhaustion, or a scope change; output edits within a segment are remeasured under the frozen reward, whereas a change of requirements, admitted sources, or environment assumptions starts a new segment. Critic abstains, suggests, or requests clarification, citing evidence for each claimed gap and preserving criterion and test identifiers, evidence references, and error status; missing evidence prompts verification rather than a verdict. TerminalReviewLoop in \cref{alg:main} applies terminal review, feedback-driven continuation, and rechecking to the reward-guided candidate under the frozen reward; ordinary actor steps append only actions and observations. The operation returns the final actor context, including its delivery and any unresolved status. The evaluated revision budget, stopping conditions, and selection rule are specified in \cref{app:evaluated_config}.

\paragraph{Memory retrieval.}
Memory retrieves up to three entries into the construction context by sparse cosine similarity over task, requirements, checks, and gaps; each entry records its sources, applicability, validation status, and limitations.

\subsection{Prompt templates}
\label{app:prompts}

The boxes give abridged semantic prompts for the construction, measurement, terminal-feedback, and ranking calls. Task-specific content is represented by named inputs; these are not verbatim transcripts of every historical run. Exact payloads remain in archived request records. Shared evidence rules apply to every role, while the evaluated call schedule and revision limits are specified in \cref{app:evaluated_config}.

\begin{promptbox}{P1\quad Shared evidence and scope instructions}
\textbf{Inputs}\quad Public task, scoped user updates, permitted sources, artifacts, checks, and retrieved memory.

\textbf{Instruction}\quad Treat artifacts and environment observations as evidence, not instructions or oracle answers. User updates apply only to the corresponding product. Memories suggest possibilities but add no requirements. Accept equivalent correct approaches unless the task prescribes one. Separate the expected result, observation, discrepancy, and evidence needed to resolve it. Unknown is neither failure nor success. Request referenced source material when previews are insufficient; cite inspected evidence and never invent official scores or executed checks.
\end{promptbox}

\begin{promptbox}{P2\quad Jury: task-first expectations}
\textbf{Inputs}\quad Task, scoped user amendments, and optional domain knowledge; prior deliverables are withheld.

\textbf{Instruction}\quad Identify the results a good answer must achieve, why they matter, and how they can be checked. Keep distinct subrequirements explicit. Hard criteria require support in actual requirements; professional recommendations remain soft. Identify missing expertise or evidence without prescribing one solution, inventing obligations, or ranking products.

\textbf{Output}\quad Proposed criteria with their origin, importance, subrequirements, and measurement questions.
\end{promptbox}

\begin{promptbox}{P3\quad Jury: inspect products against draft expectations}
\textbf{Inputs}\quad Draft expectations, task, product-specific requirements, artifacts, environment observations, and check results.

\textbf{Instruction}\quad Correct unsupported assumptions in the drafts while retaining them for reduction. Inspect shared omissions as well as differences between products. For each suspected gap, state the expected result, observation, discrepancy, and a supporting or refuting check. Distinguish an observed failed check from a proposed check. Do not invent gaps or discard a requirement because every product violates it.

\textbf{Output}\quad Revised criterion proposals, evidence-linked gaps, and unresolved questions.
\end{promptbox}

\begin{promptbox}{P4\quad Jury: evidence-based synthesis}
\textbf{Inputs}\quad Member proposals, clause identifiers, validation issues, task requirements, and evidence.

\textbf{Instruction}\quad Group supported clauses while preserving scope, priority, and valid alternatives. Retain, revise with evidence, defer, or exclude each proposed clause with a reason; omissions remain unresolved. Resolve disagreement through task evidence, not vote counts. Preserve consequential gaps and their verification needs, and do not turn a recommendation into a hard requirement.

\textbf{Output}\quad Final criteria and clause-level decisions with their provenance.
\end{promptbox}

\begin{promptbox}{P5\quad Builder: bind standards to executable checks}
\textbf{Inputs}\quad Accepted criteria, checker and measurement-tool catalogs, and permitted evidence.

\textbf{Instruction}\quad Preserve criterion identity, meaning, priority, and scope. Define fail, partial, and full anchors and propose a reusable check, splitting only when distinct execution methods require it. Prefer a parser, calculation, runtime check, or source lookup when applicable. Distinguish semantic judgment from code execution and report missing implementation capabilities. Bind checks to the current product and keep unmeasurable requirements visible.

\textbf{Output}\quad Rubric anchors and test bindings, including inputs, checker choice, and any missing implementation.
\end{promptbox}

\begin{promptbox}{P6\quad Evaluation Harness: evidence-backed measurement}
\textbf{Inputs}\quad One bound test, its standard and anchors, current product, sources, and executed tool observations.

\textbf{Instruction}\quad Evaluate only the supplied standard against observed evidence. Reading a file does not establish correctness. Request missing evidence or a supported tool; never simulate its execution. A failure needs a localized defect, a pass needs support for the material requirements, and unresolved essential evidence stays unknown. Separate advisory improvements from the verdict.

\textbf{Output}\quad A pass, fail, or unknown judgment with reasons, localized gaps, and evidence references. The Harness attaches execution receipts and error status.
\end{promptbox}

\begin{promptbox}{P7\quad Terminal Critic: review and target verification}
\textbf{Inputs}\quad Deliverable, frozen criteria, available test identifiers, open questions, and earlier findings.

\textbf{Instruction}\quad Review correctness and completeness using claims as written. Cite evidence for defects, preserve correct work, and recheck earlier findings for regressions. Prior observations are not current passes. Select targeted checks for consequential doubts, distinguishing defects from measurement errors and missing evidence. Keep improvement suggestions separate from established failures.

\textbf{Output}\quad Evidence-linked findings, proposed checks, remaining gaps, and review status. The controller executes checks and applies the limits in \cref{app:evaluated_config}.
\end{promptbox}

\begin{promptbox}{P8\quad Jury: rank completed candidates}
\textbf{Inputs}\quad Candidates, their actual user requirements, accessible evidence and receipts; the reducer also receives member comparisons.

\textbf{Instruction}\quad Compare outcomes rather than effort or compliance with advice. Identify supported quality differences before ranking. Do not transfer undisclosed requirements between products. Use ties when differences lack support and unknown when evidence is insufficient. Shared defects remain defects in the winner; ranking does not establish absolute completion. Resolve disagreements through evidence rather than vote counts.

\textbf{Output}\quad Ordered candidate groups, ties and unjudgeable candidates, with evidence-linked reasons; account for each candidate exactly once.
\end{promptbox}

Schema validation, checker dispatch, evidence binding, revision limits, and the task-only fallback are controller operations outside these prompts.

\FloatBarrier
\section{Theoretical Analysis}
\label{app:theory}
\subsection{Ratchet property}
\label{app:ratchet_proof}
\begin{proposition}[Ratchet]
\label{prop:ratchet}
If $R_{v+1}$ is obtained from $R_v$ by the acceptance rule \cref{eq:accept}, then $G_v\subseteq G_{v+1}$, and $|G_{v+1}|\geq|G_v|+1$ when the candidate is accepted. Hence, along such transitions, a case on which the reward once agreed with its reference label is never lost, and at most $|Q|-|G_0|$ repairs are accepted, where $Q$ is the final ledger.
\end{proposition}
\begin{proof}[Proof of \cref{prop:ratchet}]
Judgments are recorded once per version and replayed, so $G_v$ is fixed given $v$, and on $Q_v$ the recorded judgments of $R_{v+1}$ are those of $\widetilde R$ if it is accepted and those of $R_v$ otherwise. Since $Q_v\subseteq Q_{v+1}$, this gives $G_{v+1}\supseteq\widetilde G$ after acceptance and $G_{v+1}\supseteq G_v$ after rejection. After acceptance, the no-regression clause gives $\widetilde G\supseteq G_v$, so $G_v\subseteq G_{v+1}$ in both cases. The fixes-a-failure clause further supplies some $j^\ast\in F_v\cap\widetilde G$; as $R_v$ disagrees on $j^\ast$, we have $j^\ast\in G_{v+1}\setminus G_v$, so $|G_{v+1}|\geq|G_v|+1$. Each accepted repair thus adds at least one case to a set contained in $Q$, so at most $|Q|-|G_0|$ repairs are accepted.
\end{proof}
The statement concerns the set $G_v$; the agreement \emph{rate} on $Q_v$ can still fall when an attack round adds failing cases. It is a property of the acceptance rule, not of the attacker: it says nothing about cases that were never generated, and a systematically wrong reference label is ratcheted in as easily as a correct one. Round-boundary updates follow the caller's rules (\cref{app:runtime_details}) and are covered only when they pass the same gate. Held-out evaluation of the frozen reward against independent labels addresses these gaps.

\section{Experimental details}
\label{app:office50}\label{app:benchmarks}

\subsection{Benchmarks}
\label{app:benchmark_descriptions}
\subsubsection{WorkBuddy Bench Office}
\label{app:workbuddy_benchmark}
WorkBuddy Bench Office is the Office subset of a role-based benchmark whose other subsets cover Code, Web, and Security \citep{workbuddybench2026}. We use all 50 tasks of the version-1.0 Office subset with equal weight. They involve spreadsheets, CSV, PDF, and word-processing documents (24 data/file operations, 17 document operations, and nine workspace-automation tasks), and their difficulty lies in data relationships, document structure, state, and supporting evidence rather than file size. Agents run in a sandboxed Docker workspace through Harbor. The CompositeVerifier combines deterministic checks on the submitted artifacts with a model judge over extracted evidence through a task-specific weighted sum; across our 150 final receipts, the deterministic weight ranges from 0.70 to 0.95. A run passes at 80 (inclusive). These external graders are distinct from the internal RewardPack measurements. Among 450 round-candidate records, the only invalid score is the initial-round task-only score on \texttt{analyst-forecast-extract-L3-018}, archived as 72.8 with an invalid-scoring flag; we count it as zero and retain the raw value and flag. Selection returned the revision in that round, so it is unaffected.

\subsubsection{Claw-Eval}
\label{app:claw_benchmark}\label{app:claw_eval}\label{sec:claw_eval}\label{app:pass_rates}
Claw-Eval evaluates autonomous agents on tasks with executable environments and human-checked rubrics \citep{ye2026claweval,clawevaldocs2026}. It contains 300 tasks and 2,159 rubrics across nine categories (161 general, 101 multimodal, and 38 multi-turn tasks), provides sandboxed services, records trajectories, and distinguishes completion, safety, and robustness; its leaderboard uses $\mathrm{Pass}^{3}$ over three independent trials. We use 104 frozen task versions from 68 scenarios, including 36 bilingual pairs, scored on the continuous 0--100 native scale with a custom model judge; a run passes at 75 (inclusive). Each task version is evaluated over three sequential rounds of task-only, revision, and hypothesis-guided candidates, yielding 936 graded round candidates. Baseline repetition and aggregation are defined in \cref{app:baselines}. Our empirical three-run success metrics follow the aggregation in \cref{app:baselines,app:statistics}; the sequential rounds are dependent, unlike the leaderboard's independent-trial protocol.

\subsection{Baseline and compared policies}
\label{app:baselines}
Baseline uses two repetitions of three task-only runs per task (six executions). In every run, the same actor receives only the public task and its attachments. For a three-run statistic $g$ (\cref{app:statistics}) and native score $s^{(q)}_{ir}$ of task $i$ in repetition $q$ and run $r$,
\[
 g_i(\mathrm{Baseline})
 =\frac12\sum_{q=1}^{2}g\bigl(s^{(q)}_{i1},s^{(q)}_{i2},s^{(q)}_{i3}\bigr).
\]
Revision, hypothesis guidance, and Selected by reward apply $g$ to their three sequential rounds. We write $s_i(A)$ for the task-level mean score of method $A$; aggregation precedes paired resampling. Uniform choice averages the three candidates per task-round; Oracle takes their maximum benchmark score. Headroom is $(\text{policy}-\text{uniform})/(\text{oracle}-\text{uniform})$.

\subsection{Implementation details}
\label{app:evaluated_config}\label{app:runtime_provenance}
\begin{table}[htbp]
\caption{Evaluated configuration. On both benchmarks, the reward-guided revision permits at most two revisions with at most three targeted checks per review, and the task-only and hypothesis-guided candidates have delivery review disabled. Baseline repetition and aggregation are defined in \cref{app:baselines}.}
\label{tab:exp_config}
\centering\footnotesize\setlength{\tabcolsep}{4pt}
\begin{tabular}{>{\raggedright\arraybackslash}p{.15\linewidth}>{\raggedright\arraybackslash}p{.44\linewidth}>{\raggedright\arraybackslash}p{.33\linewidth}}
\toprule
Benchmark & Candidate generation and models & External grading \\
\midrule
WorkBuddy Bench Office & Three sequential iterations, each producing task-only, reward-guided revision, and hypothesis-guided candidates; results are aggregated across all three iterations. Actor: DeepSeek V4 Flash. & Official task-specific composite of deterministic rules and a model judge (rule weight 0.70--0.95). \\
Claw-Eval & Three sequential rounds of fresh candidates with same-task pool, checks, memory, and RewardPack carryover. Actor: DeepSeek V4 Flash; Jury and proposal roles also use MiniMax, Qwen, and GLM models. & Native grader with a custom DeepSeek V4 Flash judge. \\
\bottomrule
\end{tabular}
\end{table}

\paragraph{Candidates.}
Each round launches a fresh task-only actor that receives only the public task and its attachments. The reward-guided revision additionally receives the current RewardPack and terminal Critic review: a model overview proposes up to three targeted checks per review, and each nonempty review permits one continuation in the same actor session followed by a recheck, with at most two revisions. Measurement, Critic calls, and continuations consume the execution budget. Revision stops early only when the current overview reports no gaps, nothing is deferred or carried over, and the selected checks and all blocking requirements pass; unselected checks remain unverified, so stopping does not certify full-pack coverage. When the budget is exhausted, the latest artifact is kept with its unresolved status. The hypothesis-guided candidate receives guidance from a Hypothesis module but neither the RewardPack nor the candidate pool. Before each round, the module reads the public task, Jury gaps, permitted knowledge, and same-task artifacts and checks from earlier rounds, and returns fallible guidance (a method, an optional target, the expected benefit, a verification plan, and clarification questions) or a keep decision, which reduces the candidate to task-only generation. On WorkBuddy Bench Office the archive records the guidance as applied on 41 tasks, a keep decision on two, and an unknown state on seven. Because the surrounding loop accumulates candidates, measurements, memory, and Jury decisions, ``task-only'' describes the actor's prompt, not independence from the loop.

\paragraph{Selection.}
Three Jury members compare the candidates. A valid record places every candidate exactly once, in an ordered tie group or an unknown list, with artifact references and reasons; abstention is valid, whereas malformed output is not. A model reducer adjudicates at least two valid records against the evidence without counting votes, using ties for unsupported differences and unknowns for insufficient evidence. The terminal Jury ranks all accumulated candidates, but the selector chooses among the current round's three: the highest-ranked candidate strictly ahead of the task-only candidate with no known failure on a test, matched by identifier, that the task-only candidate passes, with eligible ties favoring the revision over the hypothesis-guided candidate; otherwise the task-only candidate.

\subsection{Evaluation metrics and statistical analysis}
\label{app:statistics}
Success thresholds are $\eta=80$ for WorkBuddy Bench Office and $\eta=75$ for Claw-Eval.

For a three-run group with benchmark scores $s_1,s_2,s_3$ and successes $b_r=\mathbf 1[s_r\geq\eta]$, the three-run statistics $g$ are the mean score, average pass rate, success coverage, and consistency:
\[
 \frac13\sum_{r=1}^{3}s_r,\qquad
 \frac13\sum_{r=1}^{3}b_r,\qquad
 b_1\vee b_2\vee b_3,\qquad
 b_1\wedge b_2\wedge b_3.
\]
Method-specific task-level aggregation is defined in \cref{app:baselines}. After that aggregation, task averages of coverage and consistency give $\mathrm{pass@3}$ and $\mathrm{Pass}^{3}$, respectively. Over the task set $\mathcal T$, we report task-averaged scores $\bar s(A)=|\mathcal T|^{-1}\sum_{i\in\mathcal T}s_i(A)$ and paired differences from Baseline, $\Delta(A)=|\mathcal T|^{-1}\sum_{i\in\mathcal T}\bigl(s_i(A)-s_i(\mathrm{Baseline})\bigr)$.

Paired 95\% percentile intervals use 200,000 resamples of WorkBuddy Bench Office tasks (NumPy default generator, seed 20260915) or the 68 Claw-Eval scenario clusters (seed 20261005), preserving bilingual task versions and all observations belonging to each sampled unit; sampled task versions receive equal weight. Differences of at most $10^{-8}$ count as ties for W/T/L, the two-sided sign test (which excludes ties), and the Wilcoxon signed-rank test. Scores are never rounded before comparison. Intervals are nominal and unadjusted for multiple comparisons. Claw-Eval task-level tests do not account for bilingual dependence; its uncertainty is assessed with the scenario-cluster intervals.

Relative failure reduction is $100\bigl(1-\frac{1-\bar p(A)}{1-\bar p(\mathrm{Baseline})}\bigr)$, where $\bar p(A)$ is the task-averaged pass rate of method $A$ at the reporting threshold. Both pass rates are recomputed together in each of the same 200,000 paired bootstrap draws; no draw has zero Baseline failure rate. This statistic and the absolute pass-rate difference summarize the same outcomes. The per-task gains in \cref{fig:benchmark_results}d,e are $s_i(\text{Selected by reward})-s_i(\mathrm{Baseline})$; values within $10^{-8}$ of zero are ties.

\subsection{Grading provenance and reproducibility}
\label{app:grading_provenance}
\paragraph{WorkBuddy Bench Office score provenance.} Archived deliveries requiring native scores were restored from their exact workspaces. We matched candidate identities and artifact records and verified inventory SHA-256 hashes, including six reused deliveries. On October 6, 2026, we ran the original task-specific rule checks in isolated task Docker environments and the native rubric judge with DeepSeek V4 Flash via OpenLux. No actor was rerun and no delivered artifact was revised. Native technical retries and format repair were retained; scores were not retried for quality. Source bindings, grader receipts, errors, and request/response bodies (except the first smoke run, which retains verdicts) are archived separately. This later grading introduces temporal judge and environment variation; bootstrap intervals condition on accepted scores and do not capture fresh-judge variability.

\subsection{Compute resources and costs}
\label{app:existing_costs}
The available token records cover auxiliary offline analyses, including failed attempts, rather than the complete benchmark workflow. Actor execution, reward construction, candidate generation, selection, and grading are not covered by a complete, non-overlapping cost ledger, and timing records omit some calls and retries. These records therefore support no total-cost, throughput, or equal-budget comparison.

\FloatBarrier
\section{Additional experimental results}
\label{app:additional_results}
This section supports the main results in \cref{sec:main_results}. \Cref{app:three_round_diagnostics,app:per_round,app:sensitivity} give delivery diagnostics and sensitivity analyses, and \cref{sec:selection_quality} details selection.

\subsection{Detailed main results}
\label{app:three_round_diagnostics}
\Cref{tab:delivery_diagnostics} collects paired effect estimates and diagnostic tests against Baseline. On Claw-Eval, selection has the highest mean and coverage, whereas revision has the highest all-round success (91/104 versus 87/104); their paired differences in mean, average pass rate, and $\mathrm{Pass}^{3}$ all have intervals that include zero. Relative failure reductions are 14.0\% ($[-4.17,32.74]$) on WorkBuddy Bench Office and 37.5\% ($[8.00,57.58]$) on Claw-Eval; their definition and resampling procedure are given in \cref{app:statistics}.

\begin{table}[htbp]
\caption{Paired delivery diagnostics against Baseline. Brackets are 95\% intervals from task or scenario resampling. Success-rate differences are percentage points (pp); median differences are score points. Sign and Wilcoxon tests are two-sided and exclude ties. Aggregation and statistical procedures are defined in \cref{app:baselines,app:statistics}.}
\label{tab:delivery_diagnostics}
\centering\small\setlength{\tabcolsep}{5pt}
\begin{tabular}{lrr}
\toprule
Quantity & Revision & Selected by reward \\
\midrule
\multicolumn{3}{l}{\textit{WorkBuddy Bench Office}} \\
$\Delta$ mean & +0.05 [$-$2.42, 2.27] & \textbf{+1.16} [$-$1.60, 3.96] \\
$\Delta$ avg. pass (pp) & +4.00 [$-$3.33, 11.00] & \textbf{+4.67} [$-$1.33, 11.00] \\
$\Delta\mathrm{pass@3}$ (pp) & $-$1.00 [$-$10.00, 7.00] & \textbf{+1.00} [$-$7.00, 8.00] \\
$\Delta\mathrm{Pass}^{3}$ (pp) & \textbf{+11.00} [2.00, 21.00] & +9.00 [0.00, 19.00] \\
Median $\Delta$ & +1.12 & +0.94 \\
Sign $p$ & 0.0649 & 0.0854 \\
Wilcoxon $p$ & 0.27373 & 0.06107 \\
\midrule
\multicolumn{3}{l}{\textit{Claw-Eval}} \\
$\Delta$ mean & +1.53 [$-$0.42, 3.33] & \textbf{+2.03} [0.90, 3.21] \\
$\Delta$ avg. pass (pp) & \textbf{+3.85} [0.17, 7.70] & \textbf{+3.85} [0.62, 7.28] \\
$\Delta\mathrm{pass@3}$ (pp) & +0.00 [$-$3.70, 3.43] & \textbf{+1.92} [0.00, 4.46] \\
$\Delta\mathrm{Pass}^{3}$ (pp) & \textbf{+9.13} [2.40, 16.20] & +5.29 [$-$1.49, 12.15] \\
Median $\Delta$ & +1.18 & +0.70 \\
Sign $p$ & $<10^{-4}$ & 0.0008 \\
Wilcoxon $p$ & 0.00097 & 0.00209 \\
\bottomrule
\end{tabular}
\end{table}

\FloatBarrier

\subsection{Per-round results}
\label{app:per_round}
\Cref{tab:per_round_performance} reports initial round and final round means and pass rates at the same thresholds as the main results. Each displayed mean uses one score per task and method in that round. Task-only denotes the fresh candidate generated from the public task in the corresponding sequential round. These rows are round-level diagnostics; the primary aggregates are defined in \cref{app:baselines,app:statistics}. 

\begin{table}[htbp]
\caption{Initial round and final round performance on 50 WorkBuddy Bench Office tasks and 104 Claw-Eval tasks. Pass rates use thresholds 80 and 75, respectively.}
\label{tab:per_round_performance}
\centering\small\setlength{\tabcolsep}{5pt}
\begin{tabular}{lrrrrrr}
\toprule
& \multicolumn{2}{c}{Task-only candidate} & \multicolumn{2}{c}{Revision} & \multicolumn{2}{c}{Selected by reward} \\
\cmidrule(lr){2-3}\cmidrule(lr){4-5}\cmidrule(lr){6-7}
Round & Mean & Pass (\%) & Mean & Pass (\%) & Mean & Pass (\%) \\
\midrule
\multicolumn{7}{l}{\textit{WorkBuddy Bench Office}} \\
Initial & 77.67 & 64.00 & \textbf{81.67} & \textbf{70.00} & 81.46 & 66.00 \\
Final & 79.41 & 62.00 & 81.30 & 70.00 & \textbf{83.52} & \textbf{76.00} \\
\midrule
\multicolumn{7}{l}{\textit{Claw-Eval}} \\
Initial & 91.24 & 94.23 & 91.13 & 92.31 & \textbf{92.51} & \textbf{97.12} \\
Final & 89.93 & 90.38 & \textbf{91.43} & \textbf{94.23} & 90.80 & 90.38 \\
\bottomrule
\end{tabular}
\end{table}
\FloatBarrier

\subsection{Sensitivity analyses}
\label{app:sensitivity}
\subsubsection{WorkBuddy Bench Office score sensitivity}
\label{app:score_sensitivity}
\Cref{tab:six_baseline_wb_sensitivity} reports three-round comparisons with Baseline after excluding the grader-exception task. The exclusion is applied to every method and changes the task population, including removing the largest selection loss; the full-sample comparison remains primary.

\begin{table}[htbp]
\caption{WorkBuddy Bench Office sensitivity against Baseline after excluding the grader-exception task. All methods retain the same exclusion; 95\% intervals resample whole tasks. Detailed pass differences and paired tests are included in the accompanying statistical output.}
\label{tab:six_baseline_wb_sensitivity}
\centering\small\setlength{\tabcolsep}{3pt}
\fitwidth{%
\begin{tabular}{ll}
\toprule
Method & $\Delta$ mean [95\% CI] \\
\midrule
\multicolumn{2}{l}{\textit{Omit grader exception}} \\
Revision & +0.74 [$-$1.23, 2.67] \\
Selected by reward & \textbf{+1.89} [$-$0.40, 4.43] \\
\bottomrule
\end{tabular}
}
\end{table}

A scan of the 150 accepted final-round receipts for explicit evaluator-exception diagnostics found one affected task-only candidate; the scan does not revalidate every grader or artifact. Two accepted final-round scores equal zero, one with missing delivery evidence and no explicit exception and the other with both an exception and missing evidence for the model judge; both remain in the full-sample analysis.

\begin{table}[htbp]
\caption{Claw-Eval results in 104 versions in 67 scenarios. All metrics use the primary aggregation defined in \cref{app:baselines,app:statistics}.}
\label{tab:claw_eval_three_round_sensitivity}
\centering\small
\fitwidth{%
\begin{tabular}{lrrrr}
\toprule
Method & Mean & Avg. pass (\%) & $\mathrm{pass@3}$ (\%) & $\mathrm{Pass}^{3}$ (\%) \\
\midrule
Baseline & 89.34 & 89.64 & 98.06 & 78.16 \\
Revision & 90.88 & \textbf{93.53} & 98.06 & \textbf{87.38} \\
Selected by reward & \textbf{91.39} & \textbf{93.53} & \textbf{100.00} & 83.50 \\
\bottomrule
\end{tabular}
}
\end{table}

\FloatBarrier

\subsection{Selection analysis}
\label{sec:selection_quality}
\Cref{tab:selection_quality} gives the full policy comparison on the generated candidate pools; \cref{tab:reward_utility_summary} summarizes gains over uniform choice. Relative to always returning the revision, selection gains $+1.11$ ($[-0.34,2.75]$) on WorkBuddy Bench Office and $+0.50$ ($[-1.01,2.21]$) on Claw-Eval, so neither interval establishes superiority over that fixed policy. Selection finds a best candidate in 74/150 and 183/312 task-rounds, versus 73/150 and 206/312 for revision. Thus, on Claw-Eval it hits a best candidate less often but has lower average regret; revision recovers 10.13\% and 31.03\% of the uniform-to-oracle gap on the two benchmarks. These comparisons characterize the selector as a whole and exclude candidate-generation costs: the archive does not identify how often rankings relied on receipts rather than artifact text. A compute-matched comparison with repeated sampling and a generic model selector remains untested.

\begin{table}[htbp]
\caption{Selection policies across all three sequential rounds. Intervals use 200,000 paired resamples of 50 WorkBuddy Bench Office tasks or 68 Claw-Eval scenarios, retaining each task's three rounds and all candidate scores. Uniform choice and Oracle operate on the same generated candidate pools as the other reported policies. Best hits include ties within $10^{-8}$; the uniform hit rate is an expectation.}
\label{tab:selection_quality}
\centering\footnotesize\setlength{\tabcolsep}{3pt}
\fitwidth{%
\begin{tabular}{lrlrr}
\toprule
Policy & Score & $\Delta$ vs. reward-guided revision [95\% CI] & Hit rate (\%) & Regret \\
\midrule
\multicolumn{5}{l}{\textbf{WorkBuddy Bench Office}} \\
Always reward-guided revision & 81.10 & +0.00 [0.00, 0.00] & 48.67 & 6.40 \\
Uniform expectation & 80.38 & -0.72 [-2.13, 0.74] & 38.22 & 7.12 \\
Selected by reward & \textbf{82.21} & +1.11 [-0.34, 2.75] & \textbf{49.33} & \textbf{5.29} \\
Oracle & 87.50 & +6.40 [3.83, 9.73] & 100.00 & 0.00 \\
\multicolumn{5}{l}{\textbf{Claw-Eval}} \\
Always reward-guided revision & 90.97 & +0.00 [0.00, 0.00] & \textbf{66.03} & 4.13 \\
Uniform expectation & 89.11 & -1.86 [-3.04, -0.51] & 48.29 & 5.99 \\
Selected by reward & \textbf{91.47} & +0.50 [-1.01, 2.21] & 58.65 & \textbf{3.63} \\
Oracle & 95.10 & +4.13 [2.59, 5.97] & 100.00 & 0.00 \\
\bottomrule
\end{tabular}
}
\end{table}

\FloatBarrier

\end{document}